\documentclass{article}

    \PassOptionsToPackage{numbers, compress}{natbib}

\usepackage[final]{neurips_2024}

\usepackage[utf8]{inputenc} 
\usepackage[T1]{fontenc}    
\usepackage{hyperref}       
\usepackage{url}            
\usepackage{booktabs}       
\usepackage{amsfonts}       
\usepackage{nicefrac}       
\usepackage{microtype}      
\usepackage{xcolor}         

\usepackage{amsmath, amssymb, amsthm, verbatim, mathtools}
\usepackage{graphicx}
\usepackage{appendix}
\usepackage{subcaption}
\usepackage{enumitem}
\usepackage{wrapfig}
\usepackage{multirow}
\usepackage{tabularx}
\newtheorem{innercustomgeneric}{\customgenericname}
\providecommand{\customgenericname}{}
\newcommand{\newcustomtheorem}[2]{%
  \newenvironment{#1}[1]
  {%
   \renewcommand\customgenericname{#2}%
   \renewcommand\theinnercustomgeneric{##1}%
   \innercustomgeneric
  }
  {\endinnercustomgeneric}
}

\newtheorem{prop}{Proposition}
\newtheorem{defin}{Definition}

\newtheorem*{prop*}{Proposition}
\newtheorem*{obs*}{Observation}
\newtheorem*{theorem*}{Theorem}
\newtheorem*{corollary*}{Corollary}
\newcustomtheorem{customprop}{Proposition}

\usepackage{tabularray}

\title{Simulation-Free Training of\\Neural ODEs on Paired Data}

\author{%
    \parbox{0.7\linewidth}{
    \vspace{-0.1cm}
    \centering
        \begin{tabular}{ccc}
        ~Semin Kim$^1$\thanks{Equal Contribution.}   & ~~Jaehoon Yoo$^{1*}$     & Jinwoo Kim$^{1}$      \\
        Yeonwoo Cha$^{1}$                           & Saehoon Kim$^{2}$     & Seunghoon Hong$^{1}$  \\
        \end{tabular}
    }\vspace{0.12cm}\\
    $^1$KAIST\hspace{1.5em}
    $^2$Kakao Brain
    \vspace{-0.1cm}
}

\begin{document}

\maketitle

\begin{abstract}
In this work, we investigate a method for simulation-free training of Neural Ordinary Differential Equations (NODEs) for learning deterministic mappings between paired data. 
Despite the analogy of NODEs as continuous-depth residual networks, their application in typical supervised learning tasks has not been popular, mainly due to the large number of function evaluations required by ODE solvers and numerical instability in gradient estimation.
To alleviate this problem, we employ the flow matching framework for simulation-free training of NODEs, which directly regresses the parameterized dynamics function to a predefined target velocity field.
Contrary to generative tasks, however, we show that applying flow matching directly between paired data can often lead to an ill-defined flow that breaks the coupling of the data pairs (\emph{e.g.}, due to crossing trajectories).
We propose a simple extension that applies flow matching in the embedding space of data pairs, where the embeddings are learned jointly with the dynamic function to ensure the validity of the flow which is also easier to learn.
We demonstrate the effectiveness of our method on both regression and classification tasks, where our method outperforms existing NODEs with a significantly lower number of function evaluations.
The code is available at \url{https://github.com/seminkim/simulation-free-node}.

\end{abstract}

\section{Introduction}
Continuous-depth models~\cite{bai2019deq, chen2018neuralode} have received growing attention as an alternative to deep feedforward networks composed of a stack of discrete layers.
As they approximate continuous, infinite-depth models with a constant set of model parameters, they are parameter-efficient models that can reuse their parameters across depth. 
Additionally, by controlling the number of function evaluations (NFEs) during inference, we can choose the optimal trade-off between performance and computational cost. 
This trade-off can be adjusted without retraining the model, unlike conventional neural networks that require separately trained models of different capacities to achieve a similar flexibility.

A prominent class of continuous-depth models is Neural Ordinary Differential Equations (NODEs)~\cite{chen2018neuralode}, which utilize continuous-time differential equations to describe evolutions of intermediate states.
In NODEs, a parameterized dynamics function learns the time derivative of the continuous transformation of a state.
NODEs can be interpreted as a continuous limit of residual networks~\citep{chen2018neuralode, haber2018}, and hence they offer a versatile framework similarly to ResNet~\cite{he2016resnet}.
As a result, NODEs have been successfully applied to tasks such as physically informed modeling~\citep{sanchez2019hogn, desmond2020symoden}, time series modeling~\citep{brouwer2019gruodebayes, kidger20ncde, rubanova2019latentode}, and generative modeling with normalizing flows~\citep{finlay2020rnode, grathwohl2019ffjord}.

However, despite their success in various tasks, the application of NODEs to learn deterministic mappings between paired data, such as regression or classification, remains under-explored. 
This is largely due to the substantial computational burden of training NODEs, since numerical ODE solving during training is inherently serial, slow, and requires large NFEs. 


Recently, an alternative, \emph{simulation-free} training method has been introduced for ODE-based generative models~\citep{albergo2023si, lipman2022cfm, liu2022rectifiedflow}. 
Known as flow matching, this approach eliminates the need for the expensive ODE solving during training by directly regressing the model on a presumed velocity field. 
While this method significantly improves training efficiency by requiring only one function evaluation per training step, its effectiveness on deterministic tasks has not been well explored.

In this work, we investigate a simulation-free training method for continuous-depth models in learning deterministic mapping between paired data.
We first identify potential problems that occur when applying flow matching objective to NODEs, and then propose a method to mitigate them.
With our simulation-free training scheme, we can significantly reduce the computational demands of NODE training while maintaining competitive performance in deterministic tasks. 
We demonstrate the advantages of our method on both regression and classification problems, providing empirical evidence of its effectiveness and versatility in common supervised learning settings. 
Furthermore, by leveraging simple flows, our method can achieve superior performance compared to NODEs in low-NFE regime.

\section{Preliminary}
\label{sec:preliminary}
\paragraph{NODEs as Continuous-Depth Models}
We aim to learn a deterministic mapping between paired data $\mathcal{D}=\{x^{(i)}, y^{(i)}\}_{i=1}^N$ with a continuous-depth model.
To this end, we consider Neural Ordinary Differential Equations (NODEs)~\cite{chen2018neuralode}, which can be regarded as a continuous limit of residual networks. 
Formally, a single layer of a residual network transforms a hidden state $z$ with a residual connection, \emph{i.e.} $z_{t+1} = z_{t} + h_\theta(z_{t})$. 
This update resembles a single step of Euler solver, which uses parameterized dynamics function $h_\theta$ to approximate the derivative of $z$ with respective to $t$. 
Taking this discretization to the continuous limit, residual networks are equivalent to the following ODE:
\begin{equation} \label{eqn:ode}
    v_t = \frac{\mathrm{d}z_t}{\mathrm{d}t} =  h_\theta(z_t, t).
\end{equation}
To train a NODE with a loss defined on the state at time $t_2$, we first solve an initial value problem starting from a known initial state $z_{t_1}$:
\begin{equation}\label{eqn:ode_ivp}
    z_{t_2} = z_{t_1} + \int_{t_1}^{t_2} h_\theta(z_{t}, t) dt = \texttt{ODESolve}(z_{t_1}, h_\theta, t_1, t_2).
\end{equation}
For the task of fitting a paired dataset, each data-label pair $(x, y)$ is placed on the state-space of the ODE as endpoints $(z_0, z_1)$ for time interval $[0,1]$.
This requires matching the dimensions of data $x$ and label $y$ to states $z$, which is done by projecting data $x$ to initial state $z_0=f_\phi(x)$,\footnote{A typical choice for $f$ in the literature is the identity mapping (\emph{i.e.}, $z_0=x$).}
and the solved final state $z_1$ to predicted label $\hat{y}=d_\psi(z_1)$.
Then, NODEs are trained end-to-end to minimize the loss $\mathcal{L}(\hat{y}, y)$.
However, this approach is inherently slow and computationally intensive, as it requires a large number of sequential function evaluations by the ODE solver~\cite{chen2018neuralode, kidger2022onnodes, pal2021nde}.
There is also a hidden cost in integrating the ODE backward in time for gradient estimation (\emph{i.e.}, adjoint method ~\cite{chen2018neuralode}), which necessitates additional serial function evaluations and is numerically unstable~\cite{zhuang2021mali}.
Also, common choices of adaptive-step ODE solvers tend to make NODEs learn arbitrarily complex trajectories~\cite{zhuang2020aca}, making both training and inference computationally expensive.

\paragraph{Flow Matching for Simulation-Free Training} 
Instead of optimizing Eq.~\eqref{eqn:ode} via the expensive initial value problem of Eq.~\eqref{eqn:ode_ivp}, we can alternatively employ the flow matching framework~\cite{albergo2023si, lipman2022cfm, liu2022rectifiedflow} to directly regress the dynamics function $h_\theta$ to the velocity field $v_t$ by:
\begin{equation} \label{eq:flow_matching}
    \mathbb{E}_{t} [||h_\theta(z_t, t) - v_t||_2^2].
\end{equation}
Since the intermediate state $z_t$ and its velocity $v_t$ are generally unknown and intractable to compute, existing works~\citep{albergo2023si, lipman2022cfm, liu2022rectifiedflow} employ predefined, tractable conditional velocity fields defined per sample.
Specifically, a closed form expression of $z_t$ is given as an interpolant of two endpoints and time (\emph{i.e.} $z_t=\alpha_t z_0 + \beta_t z_1$), and a target velocity field is constructed as a time derivative of the interpolant.
One simple instantiation of an interpolant and its corresponding target velocity field is \textit{linear} dynamics with a constant speed~\cite{lipman2022cfm, liu2022rectifiedflow}:
\begin{align} \label{eq:linear_ansatz}
    z_t = (1-t)z_0 + t z_1, ~~~~v_t = z_1 - z_0.
\end{align}
Since we can obtain $v_t$ at arbitrary time $t$, this allows the dynamics function to be trained in a simulation-free manner, where the regression of the vector field is performed parallel in time.
Simulation-free training in the flow matching models is intriguing, since there is no need for serialized function evaluations during training as well as backpropagation through time. 
Employing a simple velocity field such as Eq.~\eqref{eq:linear_ansatz} also greatly simplifies the trajectory of the dynamics function, reducing the number of function evaluations for inference. 
However, flow matching has been mainly studied for generative modeling, which aims to find a transportation map between two marginal distributions $p(z_0)$ and $p(z_1)$ while learning arbitrary \emph{per-sample} couplings between $z_0\sim p(z_0)$ and $z_1\sim p(z_1)$.
It makes it difficult to be applied in deterministic regression tasks where the per-sample coupling is defined by data pairs $(z_0,z_1)\sim(x,y)$.
We elaborate on this issue in the next section.

\section{Challenges in Flow Matching for Paired Data} 
\label{sec:conflict}

\begin{figure}[t]
    \centering
    \begin{subfigure}[t]{0.24\linewidth}
        \centering
        \includegraphics[width=1.0\linewidth]{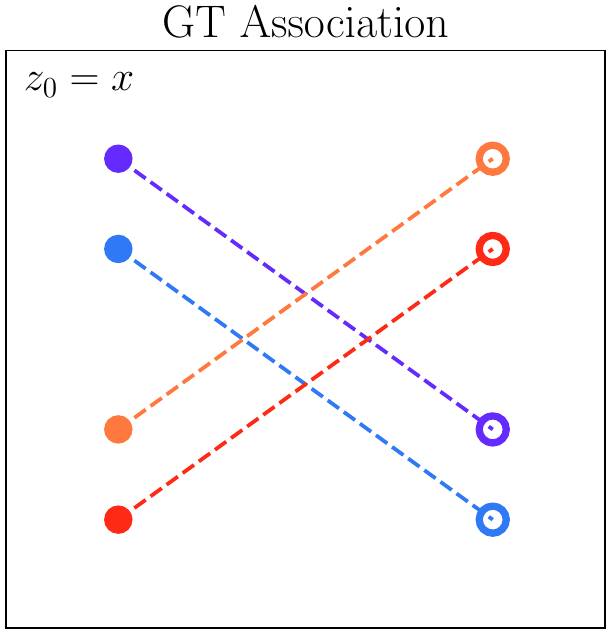}
        \subcaption{Ground Truth}
    \end{subfigure}
    \begin{subfigure}[t]{0.24\linewidth}
        \centering
        \includegraphics[width=1.0\linewidth]{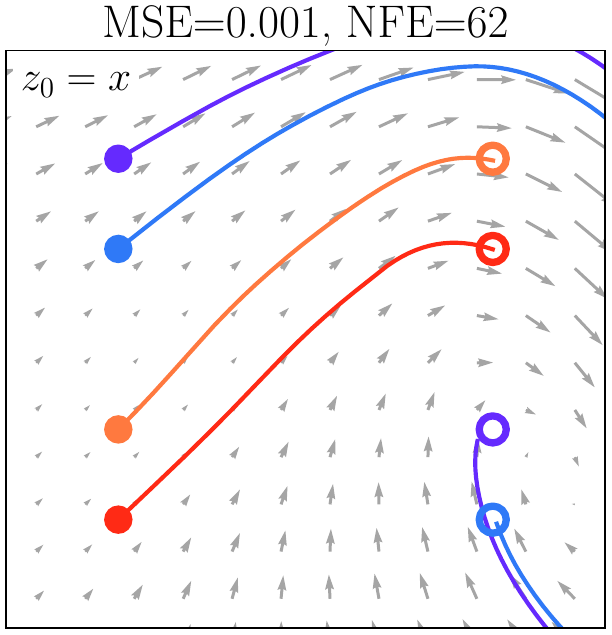}
        \subcaption{Neural ODE}
    \end{subfigure}
    \begin{subfigure}[t]{0.24\linewidth}
        \centering
        \includegraphics[width=1.0\linewidth]{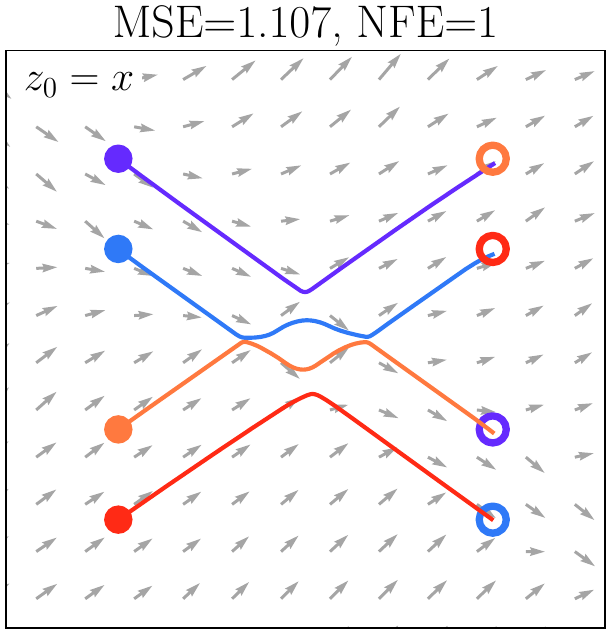}
        \subcaption{Flow Matching (FM)}
    \end{subfigure}
    \begin{subfigure}[t]{0.24\linewidth}
        \centering
        \includegraphics[width=1.0\linewidth]{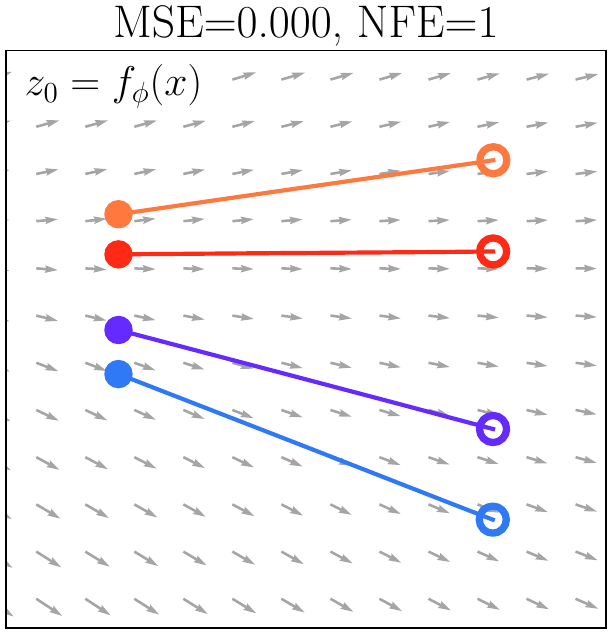}
        \subcaption{Embed. + FM (Ours)}
    \end{subfigure}
    \caption{
            Comparison of the learned trajectories. 
            The final train loss (MSE) and training NFE are shown above each plot.
            (a) We consider deterministic regression task of four data pairs, each of which is represented by two circles (filled and empty circles) connected by dotted lines. 
            (b) NODEs can correctly associate the pairs but through complex paths (solid lines) that require large NFEs. 
            (c) Flow matching with linear velocity can greatly reduce the training NFEs by simulation-free training, but fails to associate the correct pairs due to the crossing trajectories induced by predefined dynamics. 
            (d) The proposed method can alleviate the problems by learning the embeddings for data jointly with the flow matching.}
    \label{fig:toy-regression}
\end{figure}

We seek to adopt the simulation-free training objective in Eq.~\eqref{eq:flow_matching} to learn a continuous-depth model of Eq.~\eqref{eqn:ode}.  
However, we find that careful consideration is required when employing the flow-matching objective to model a deterministic mapping between paired data. 

To illustrate this point, we provide a toy example in Fig.~\ref{fig:toy-regression}.
Here, we consider a simple regression task with two-dimensional inputs and outputs, translating four points at $x=-1$ to the ones at $x=1$ but in a different order in the $y$ axis. 
For ease of analysis, we consider simple linear velocity field in Eq.~\eqref{eq:linear_ansatz} for flow matching.

First, we observe that NODE successfully learns to associate inputs and outputs by solving Eq.~\eqref{eqn:ode_ivp} (Fig.~\ref{fig:toy-regression}~(b)).
However, the learned trajectories are highly non-linear and complex which leads to a large number of function evaluations for training and inference. 

The flow matching can greatly improve the computation cost by optimizing the simulation-free objective in Eq.~\eqref{eq:flow_matching}.
However, we observe that it fails to infer the correct input-output correspondence according to the learned trajectories (Fig.~\ref{fig:toy-regression}~(c)).
This is because some target trajectories induced by the predefined velocity field are crossing each other, which cannot be modeled by ODEs~\cite{younes2010shapes}.
As a result, the learned dynamics function at the intersection of crossing trajectories produces their mean velocity, which makes the inferred trajectory fail to arrive at the correct output.



Note that the crossing trajectories induced by the predefined flow are less problematic in generative tasks. 
It is because their objective is to transport between two marginal distributions $p(z_0)$ and $p(z_1)$ while allowing arbitrary association between samples $z_0\sim p(z_0)$ and $z_1\sim p(z_1)$.
It is evident from Fig.~\ref{fig:toy-regression}~(c), where the model fails to preserve the initial coupling in the dataset but still yields valid marginal distribution at the output. 
However, it is not desirable for deterministic regression, since the goal is to preserve \emph{per-sample} correspondence between data and label.

\section{End-to-End Latent Flow Matching} \label{sec:method}
Previous discussions suggest that predefined flow between paired data can yield crossing trajectories that cannot be modeled by the well-defined ODE, which leads to invalid paths by the dynamics function that break the coupling of data and label. 
As a simple solution, we propose to learn the embeddings of data and label jointly with dynamics function, in such a way that the predefined flow induces non-crossing trajectories in the embedding space (Fig.~\ref{fig:toy-regression}~(d)). 

Fig.~\ref{fig:overview} illustrates the overall framework. 
The proposed framework comprises the data encoder $f_\phi$, label encoder $g_\varphi$, label decoder $d_\psi\approx g_\varphi^{-1}$, and dynamics function $h_\theta$.
Given a data pair $(x,y)$, our model first projects the data and label using respective encoders by $(z_0,z_1)=(f_\phi(x), g_\varphi(y))$. 
Then, for all $t\in[0,1]$, the state $z_t$ and the corresponding target velocity $v_t$ are obtained by some predefined dynamics (\emph{e.g.,} Eq.~\eqref{eq:linear_ansatz}). 
We train the encoders and dynamics function jointly using the flow matching loss, while an additional label autoencoding objective is applied to the label encoder and decoder. 
The inference procedure remains similar to that of NODEs: given the input embedding $z_0=f_\phi(x)$, we obtain the state $\hat{z}_1$ by solving ODE in Eq.~\eqref{eqn:ode_ivp} and decode it to a label by the label decoder $\hat{y}=d_\psi(z_1)$.
Below, we describe the learning objective and optimization of our method in detail.

\begin{figure}[t]
\begin{center}
\includegraphics[width=0.7\linewidth]{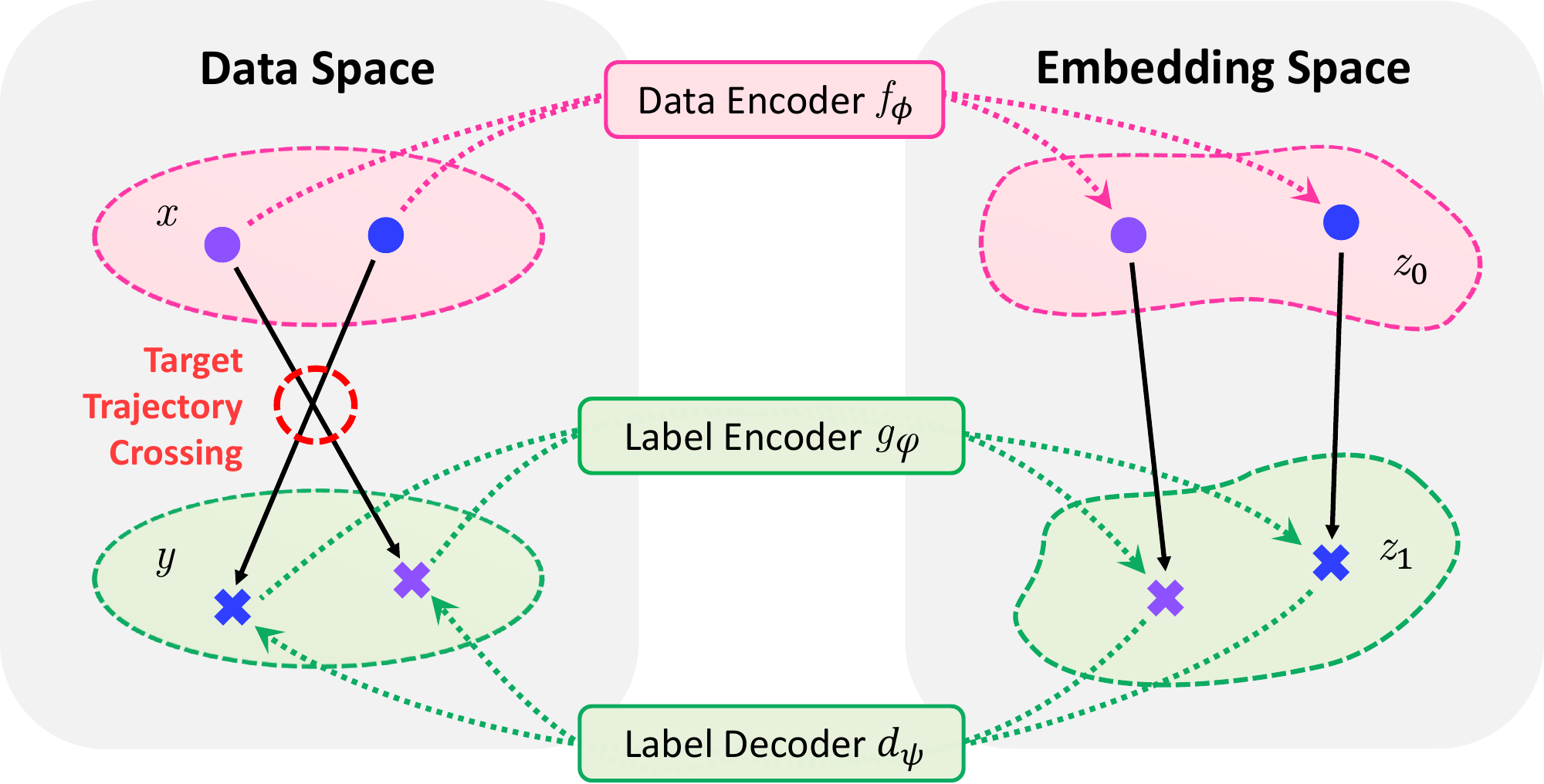}
\end{center}
\caption{
An overview of our framework. We avoid the crossing trajectory problem in data space by introducing learnable encoders that project data and label to embedding space. In the learned embedding space, the presumed dynamics induce valid target velocity field.
}
\label{fig:overview}
\end{figure}

\paragraph{Flow Loss}
With the learnable projections $z_0=f_\phi(x), z_1=g_{\varphi}(y)$ and any simple dynamics assumption $z_t=\alpha_t z_0 + \beta_t z_1$, our flow matching objective is defined by:
\begin{equation} \label{eq:flow_loss}
\mathcal{L}_{flow}(\theta, \phi, \varphi) = \mathbb{E}_{t, (x,y)\sim\mathcal{D}} [||h_\theta(z_t, t) - v_t||_2^2].
\end{equation}

In contrast to conventional flow matching (Eq.~\eqref{eq:flow_matching}) that optimizes the dynamics function given the fixed endpoints $z_0$ and $z_1$, the loss in Eq.~\eqref{eq:flow_loss} is optimized jointly with two encoders $f_\phi$ and $g_\varphi$. 
Since the loss has the optimum at $\mathcal{L}_{flow}(\theta^*,\phi^*,\varphi^*)=0$ when the dynamics function $h_{\theta^*}$ perfectly fits to the velocity field $v_t$, optimizing Eq.~\eqref{eq:flow_loss} guarantees that the optimal encoders $f_{\phi^*}$ and $g_{\varphi^*}$ induce non-crossing target trajectories in the embedding space.

Conceptually, our approach resembles the \textit{reflow} procedure in Rectified Flow~\cite{liu2022rectifiedflow}, which straightens the target trajectories by recursively rewiring $z_0$ and $z_1$ according to the trajectory obtained by solving an ODE. 
In contrast, our approach straightens the trajectory by learning the encoders such that the trajectories defined by the predefined flow do not cross in the embedding space, which does not involve ODE solving and, most importantly, preserves the initial coupling $(z_0,z_1)$.

However, contrary to Eq.~\eqref{eq:flow_matching}, we observe that optimizing the flow matching objective with learnable encoders in Eq.~\eqref{eq:flow_loss} can lead to trivial degenerate solutions.
One such example is when both encoders collapse to output constant ignoring data, which produces trivial targets for the dynamic model (\emph{e.g.}, $z_0=z_1=h_\theta(\cdot, \cdot)=0$).
To avoid trivial solutions, both encoders require strong regularization to be relevant to the inputs. 
Fortunately, such regularization is naturally provided by the learning objective for the label decoder $d_\psi$, which is described below.

\paragraph{Label Autoencoding Loss}
We define the \textit{label autoencoding loss} as a simple mean squared error (MSE) between the label $y$ and its reconstruction obtained by label encoder and decoder:
\begin{equation} \label{eq:label_ae_loss}
    \mathcal{L}_{label\_ae}(\psi, \varphi) = \mathbb{E} [||d_{\psi}( g_{\varphi}(y) ) - y||_2^2].
\end{equation}
The primary purpose of Eq.~\eqref{eq:label_ae_loss} is to provide learning signals for the label decoder $d_\psi$, which is used to decode the predicted $z_1$ to the label $\hat{y}=d_\psi(z_1)$ at inference.
However, Eq.~\eqref{eq:label_ae_loss} also functions as a regularization to avoid trivial solutions of Eq.~\eqref{eq:flow_loss} by preventing the label encoder $g_\varphi$ from ignoring the input label $y$. 
Regularizing the label encoder is sufficient to prevent trivial solutions in principle, since it eliminates the trivial targets for both data encoders $f_\phi$ and the dynamics function $h_\theta$.



\paragraph{Objective Function}

The overall objective function for the proposed framework is given by combining the two losses introduced above:
\begin{equation} \label{overall_objective}
    \min _{\theta, \phi, \varphi, \psi} \mathcal{L}_{flow}(\theta, \phi, \varphi) + \mathcal{L}_{label\_ae}(\psi, \varphi)
\end{equation}
Theoretically, we can show that optimizing the combination of flow loss and autoencoding loss in Eq.~\ref{overall_objective} can indeed mitigate the problem of target trajectory crossing:


\begin{prop}
\label{prop:1}
    There exist $f_\phi$ and $g_\varphi$ that induce non-crossing target trajectory for any data pair in $\mathcal{D}$ while minimizing $\mathcal{L}_{label\_ae}$.
\end{prop}
\begin{proof}
    The proof can be found in App.~\ref{app:proof_prop1}.
\end{proof}

\begin{prop}
\label{prop:2}
    If $g_\varphi$ is injective, the following equivalence holds: $f_\phi,g_\varphi$ and $h_\theta$ minimize $\mathcal{L}_{flow}$ to zero for all $t\in[0, 1)$ if and only if $f_\phi$ and $g_\varphi$ induce non-crossing target trajectory and $h_\theta$ perfectly fits the induced target velocity, for any data pair in $\mathcal{D}$.
\end{prop}
\begin{proof}
    The proof can be found in App.~\ref{app:proof_prop2}.
\end{proof}

Proposition~\ref{prop:1} ensures the existence of encoders that do not induce crossing in the target trajectory, while Proposition~\ref{prop:2} suggests that the encoders are optimal when flow loss is minimized, assuming the label encoder is injective. 
Since the label autoencoding task enforces $g_\varphi$ to be injective, combining the two propositions implies that optimizing the objective function in Eq.~\eqref{overall_objective} prevents the encoders from inducing target trajectory crossing and enables the dynamics function to accurately fit the induced trajectories.

\paragraph{Optimization} \label{optimization_technique}
During the optimization of Eq.~\eqref{overall_objective}, we find that two additional regularizations on the encoders are useful in further preventing suboptimal solutions and stabilizing training and inference.

First, although the autoencoding loss in Eq.~\eqref{eq:label_ae_loss} prevents degenerate encoders in principle, we find that flow matching loss in Eq.~\eqref{eq:flow_loss} can still induce an ill-behaved local minima for the data encoder $f_\phi$.
For instance, under linear velocity field (Eq.~\eqref{eq:linear_ansatz}), a collapsed, constant data encoder, \emph{e.g.}, $z_0=0$, makes both intermediate state and target velocity depend only on $z_1$ at $t\in(0,1]$, \emph{e.g.}, $z_t = t z_1$ and $v_t = z_1$. 
In this case, a dynamics function that scales with time, $h_\theta(z,t) = z/t$, can fit the target velocity field for all $t \in (0,1]$ although it does not yield meaningful mapping between $(x, y)$.
We find that explicitly sampling $t=0$ with a certain probability during training effectively resolves the issue.

Second, we empirically find that the label encoder tends to reduce the scale of the output embeddings to optimize the flow matching loss.
Although this is not a fundamental problem in principle, we observe that it often affects the generalization performance by making the model prone to small numerical errors at inference, such as prediction errors in the dynamics function or discretization errors in the ODE solver.
To address this, we encourage the label embeddings to \textit{repel} each other, so that they construct a robust destination point for ODE solving during inference.
Specifically, during training we add random noise $\epsilon \sim \mathcal{N}(0,\sigma^2)$ to label embedding and let the label decoder to reconstruct  original label from it, \emph{i.e.},  we minimize $\mathbb{E} [||d_{\psi}( g_{\varphi}(y) + \epsilon ) - y||_2^2]$ instead of Eq.~\eqref{eq:label_ae_loss}.
We empirically find these techniques helpful for better optimization, as shown in Sec.~\ref{ablation_collapse}.

\section{Related Work}
\paragraph{Enhancing Efficiency of NODEs}

Several previous works have addressed the problem of high numbers of function evaluations (NFEs) in the forward process of Neural Ordinary Differential Equations (NODEs). 
Most approaches involve imposing regularization on the learned trajectory, such as penalizing higher-order derivatives~\cite{kelly2020taynode} or incorporating kinetic regularization~\cite{finlay2020rnode}. 
Similar effects can be achieved through weight decay~\cite{grathwohl2019ffjord}, augmenting dimensions~\cite{dupont2019anode}, or using internal solver heuristics~\cite{pal2021nde}. 
Additionally, sampling the end time of the integration interval~\cite{arnab2020steer} has also been explored as a simple solution. 
These methods encourage the model to learn simpler trajectories, thereby effectively reducing the NFE required to solve the ODE.
However, with the common choice of adaptive-step solvers, there often remain dozens of sequential function evaluations during a single training step, making the training of NODEs slow and computationally intensive.

\paragraph{Simulation-Free Training on Paired Dataset}

Some studies have explored simulation-free training methods for fitting dynamics functions on paired datasets, primarily within the context of diffusion probabilistic models (DPMs)~\cite{ho2020ddpm}. 
DPMs achieve simulation-free training by learning to denoise data at multiple noise levels in parallel~\cite{luo2022understanding}.
During inference, DPMs generate data from standard Gaussian noise through iterative denoising. 
These models have been applied to various vision tasks aimed at learning deterministic mappings, including segmentation~\citep{amit2021segdiff, baranchuk2022labeffsemseg, chen2023panopticseg}, object detection~\cite{chen2023diffusiondet}, and image restoration tasks~\citep{li2022srdiff, rombach2022stable, whang2022deblurring}, among others.
While these applications are impressive, they are tailored specifically to their target tasks and do not represent general methods for paired data with diverse label structures.

One notable work in this area is CARD~\cite{han2022card}, which introduced a new conditional diffusion process for classification and regression tasks, making it applicable to arbitrary regression and classification data. 
Although these works, based on diffusion models trained with denoising objectives, align more closely with neural SDEs rather than ODEs, we include a comparison with CARD in our main experiments to provide a comprehensive evaluation.

\section{Experiments}
\label{sec:experiments}

\subsection{Experimental Setup}
\paragraph{Baselines and Datasets}
We validate the effectiveness of our method on various datasets for both regression and classification tasks. 
For baselines, we compare our method with the standard NODE~\cite{chen2018neuralode} and previous works that utilize regularization to reduce NFEs. 
Specifically, we compare against STEER~\cite{arnab2020steer}, which introduces stochasticity in the integration interval, and RNODE~\cite{finlay2020rnode}, which regularizes the norm of the velocity field. 
Additionally, we include a comparison with CARD~\cite{han2022card}, a classification and regression model based on diffusion. 
Following prior work~\citep{dupont2019anode, arnab2020steer}, we use MNIST~\cite{mnist}, SVHN~\cite{svhn}, and CIFAR10~\cite{cifar10} for image classification experiments. 
For regression tasks, we use UCI regression datasets~\cite{uci}, adhering to the protocol used by CARD.

\paragraph{Evaluation}
We report classification accuracy and root mean square error (RMSE) as the main performance metrics. 
To quantify computational costs, we report average per-sample NFEs along with training throughput measured by the total training iterations divided by training time.
For all NODE baselines, we use the dopri5~\cite{dopri5} adaptive-step solver implemented in torchdiffeq~\cite{torchdiffeq} package for both training and inference. 
Additionally, we report few-step inference results using the Euler solver. 
For CARD, we perform few-step inference by periodically skipping intermediate steps, similar to DDPM in few-step inference~\citep{ho2020ddpm, song2021ddim}, and report the metric of full-step inference in the place of adaptive-step solver in NODEs.

\paragraph{Implementation Details}
We use the same network architecture across baselines, employing an MLP-based architecture for MNIST and UCI and a convolutional architecture for SVHN and CIFAR10. 
For classification tasks, we use one-hot encoded labels for $y$ and assign the predicted label $\hat{y}$ by applying argmax on the channel dimension. 
To handle the significant memory requirements of training NODE-based baselines, we use the adjoint sensitivity method~\cite{chen2018neuralode} for SVHN and CIFAR10 experiments. 
When using the adjoint method, we report the total NFE by summing the number of function evaluations in both the forward and backward passes. 
For our model, we primarily use a simple linear dynamics assumption, which is shown to be effective according to the analysis in Sec.~\ref{sec:dynamics_type}. Further experimental details can be found in App.~\ref{app:experiment_details}.


\begin{figure*}[!t]
\centering
    \begin{minipage}[!t]{1.0\textwidth}
    \centering
    \captionof{table}{
    Experiment results on image classification.
    Training cost and few-/full-step performances are reported in three datasets. 
    For classification accuracy, numbers indicate the number of function evaluations with Euler solver, where $\infty$ denotes the result of dopri5 adaptive-step solver. For CARD, we report the 1000-step decoding results instead of using the adaptive solver, as the model was trained on discrete timesteps. CARD\textsuperscript{\textdagger} is trained with 4 times longer steps.
    }
    \label{tab:classification_result_main}
    \footnotesize
    \begin{tabular}{ccccccccc}
    \toprule
    && \multicolumn{2}{c}{Training Cost} & \multicolumn{5}{c}{Accuracy over NFE (\%)} \\
    \midrule
    \multirow{2}{*}{Dataset} & \multirow{2}{*}{Model} & \multirow{2}{*}{NFE}      & Throughput      & \multirow{2}{*}{1}          & \multirow{2}{*}{2}   & \multirow{2}{*}{10}  &  \multirow{2}{*}{20} & \multirow{2}{*}{$\infty$} \\
    & & & (Batch / sec.) & & & & & \\
    \midrule
    \multirow{5}{*}{MNIST}
    & NODE         & 354 &  0.93  & 14.20          & 7.71           & 8.81           & 10.83          & 98.36               \\
    & STEER        & 90   & 3.22  &         24.79       &     29.39       &    49.57       &   62.68        &    99.23   \\
    & RNODE      & 43  &  3.27   &      \textbf{99.35}    &      \textbf{99.35}       &     \textbf{99.35}      &    \textbf{99.35}       &     \textbf{99.35}    \\
    & CARD & \textbf{1} & 28.95 &      9.72        &       15.88         & 98.92 & 99.12 & 98.83  \\
    \cmidrule(l){2-9} 
    & Ours         & \textbf{1} & \textbf{29.41} & 99.33 & 99.25 & \textbf{99.35} & 99.34 & 99.30  \\
    \midrule
    \multirow{6}{*}{SVHN}
    & NODE        & 75 & 0.42 & 87.72 & 91.89 & 95.18 & 95.16 & 95.09\\
    & STEER     & 110 & 0.27 & 30.55 & 65.00 & 93.22 & 94.44 & 94.55\\
    & RNODE     & 130 & 0.14 & 93.36 & 95.01 & 95.37 & 95.39 & 95.39\\
    & CARD      & \textbf{1} & 7.94 & 9.95 & 16.01 & 76.87 & 79.49 & 65.31 \\
    & CARD\textsuperscript{\textdagger}      & \textbf{1} & 7.94 & 9.95 & 35.91 & 95.36 & 95.31 & 95.23 \\
    \cmidrule(l){2-9} 
    & Ours      & \textbf{1} & \textbf{9.48} & \textbf{96.16} & \textbf{96.14} & \textbf{96.03} & \textbf{96.03} & \textbf{96.12} \\
    \midrule
    \multirow{6}{*}{CIFAR10}
    & NODE      & 91 & 0.34 & 81.18 & 82.72 & 86.25 & 86.33 & 86.30 \\
    & STEER     & 96 & 0.33 & 76.80 & 77.78 & 83.55 & 84.24 & 84.51\\
    & RNODE     & 89 & 0.19 & 79.61 & 81.62 & 85.68 & 85.96 & 86.08 \\
    & CARD      & \textbf{1} & 7.67 & 10.22 & 18.78 & 84.48 & 84.68 & 81.77 \\
    & CARD\textsuperscript{\textdagger} & \textbf{1} & 7.67 & 10.61 & 33.69 & 86.67 & 86.54 & 86.42 \\
    \cmidrule(l){2-9} 
    & Ours & \textbf{1} & \textbf{9.00} & \textbf{88.87} & \textbf{88.85} & \textbf{88.71} & \textbf{88.88} & \textbf{88.89} \\
    \bottomrule
    \end{tabular}%
    \end{minipage}
    \\
    \vspace{0.1cm}
    \begin{minipage}[!t]{1.0\textwidth}
    \centering
    \includegraphics[width=1.0\linewidth]{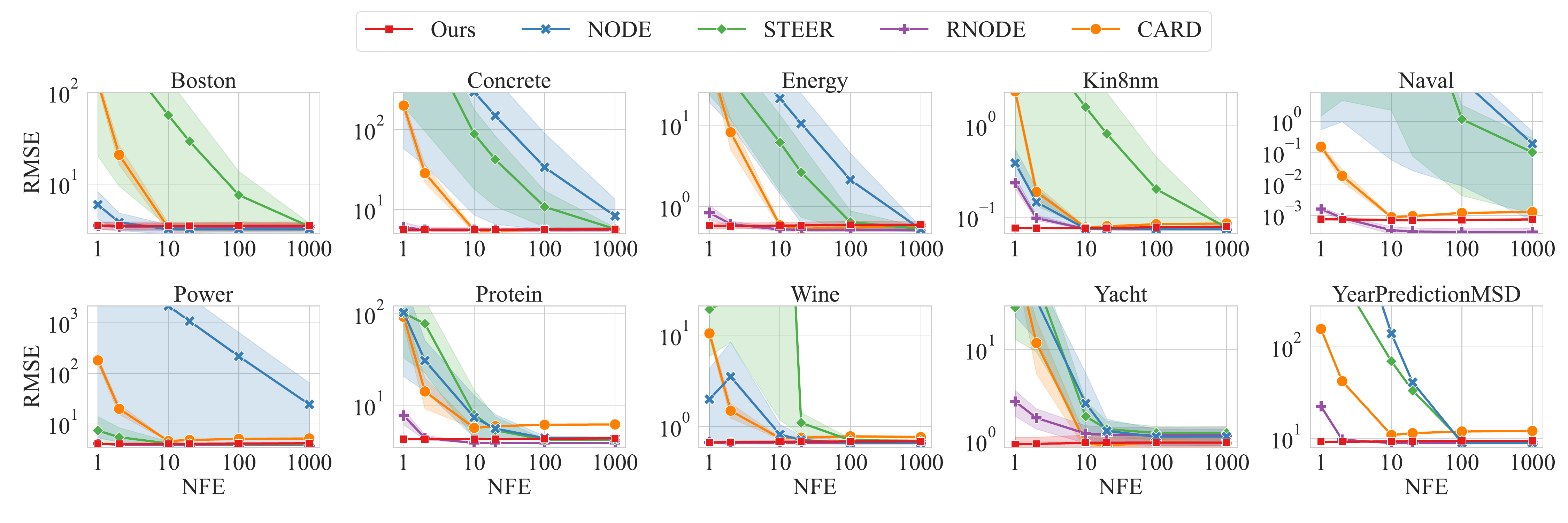}
    \vspace{-0.7cm}
    \captionof{figure}{RMSE over NFEs on UCI regression tasks. To control the NFE, we use Euler solver for the evaluation. By assuming linear dynamics, our model shows better performance in low NFE regime.}
    \vspace{-0.5cm}
    \label{fig:regression_result_main}
    \end{minipage}
\end{figure*}

\subsection{Main Results}
\paragraph{Training Cost and Performance}
In Tab.~\ref{tab:classification_result_main}, we report the training cost and classification accuracy on MNIST, SVHN, and CIFAR10. 
NODE-based baselines suffer from large NFEs—ranging from tens to hundreds—during training, resulting in a low throughput. 
In contrast, simulation-free training method (\emph{i.e.}, ours and CARD) require only a single function evaluation per training step, significantly boosting training speed.

The trade-off for faster training in simulation-free methods is a constraint on the dynamics that can be learned. 
However, we observe that this reduced flexibility does not lead to significant performance degradation. 
Compared to NODE-based baselines, our method shows only minor degradation on MNIST and even improvements on SVHN and CIFAR10 in terms of final classification accuracy. 
On SVHN and CIFAR10, NODE-based baselines are trained with the adjoint sensitivity method~\cite{chen2018neuralode} to meet memory requirements, which is known to suffer from inaccurate gradient estimation~\citep{zhuang2020aca, zhuang2021mali}. 
Thus we hypothesize that the performance gains on SVHN and CIFAR10 may be due to the accurate gradient calculation achieved through direct backpropagation in our model, which is a positive byproduct of utilizing simulation-free training.

When compared to the diffusion-based baseline, our model tends to show better performance and faster convergence.
Specifically, our model outperforms CARD on all three datasets, even compared to CARD variants trained for longer iterations. 
While diffusion models also benefit from simulation-free training, they are based on stochastic differential equations that induce stochastic and non-linear trajectories.
We conjecture that this characteristic is not ideal for few-step inference and also contributes to slower convergence.

\paragraph{Few-Step Inference}
Ideally, our model trained with linear dynamics will yield a perfectly straight solution trajectory, which can be accurately estimated even with a one-step Euler solver. 
Consequently, with our linear dynamics assumption, we can significantly enhance inference speed by utilizing few-step inference while maintaining competitive performance compared to many-step solving. 
To demonstrate this, we report few-step inference results with the Euler solver in Tab.~\ref{tab:classification_result_main}. 
Our model, by avoiding crossing points in the learnable embedding space, produces a linear trajectory and thus exhibits superior few-step performance. 
This observation aligns with findings in flow matching models~\citep{lipman2022cfm, liu2022rectifiedflow}, which highlight the advantage of linear dynamics for generating high quality samples with low inference cost.

\paragraph{Regression}
We further analyze the effectiveness of our method in regression tasks, as shown in Fig.~\ref{fig:regression_result_main}. 
See App.~\ref{app:regression_full} for full results.
Despite differences in label structure, we observe similar trends for both classification and regression tasks: our method significantly reduces computational burden during training and demonstrates superior performance in few-step inference with linear dynamics. 
Similar to the original NODE, our model can be effectively applied to a wide range of common supervised learning settings, regardless of whether the labels are categorical or continuous.

\subsection{Analysis and Discussion} \label{sec:analysis}


\begin{wrapfigure}{r}{0.48\textwidth}
\footnotesize
    \centering
    \captionof{table}{
    The effectiveness of learning encoders with flow loss.
    Training accuracy and the proportion of disagreement in prediction between a one-step Euler solver and an adaptive-step solver are shown.
    Simply augmenting dimensions (ANODE+FM) does not effectively prevent trajectory crossing.
    Furthermore, learning encoders without flow loss (Autoencoder+FM) also fails to preserve the original coupling due to crossing trajectories.
    }
    \label{tab:twostage}
    \begin{tabular}{c|cc}
    \toprule
    Training & Disagreement & Accuracy \\
    \midrule
        ANODE + FM          & 47.30\%           & 30.23\%           \\
        Autoencoder + FM    & 11.56\%           & 55.73\%           \\
        Ours                & \textbf{0.02\%}   & \textbf{99.80\%}  \\
    \bottomrule
    \end{tabular}
\end{wrapfigure}

\paragraph{Learning Encoders with Flow Loss}
To support our key claim that learning encoders with flow loss allows our model to avoid crossing trajectories, we compare our model with  two variants that do not utilize flow loss for learning encoders.
Specifically, we consider: (1) \textit{ANODE+FM} which augments the data and label to same dimensionality by zero-padding and learns dynamics function with flow matching; and (2) \textit{Autoencoder+FM} which employs a two-stage approach by first learning the embedding space using an independent autoencoding objective for both data and labels, and then learning the dynamics function on the fixed embedding space\footnote{The samples reconstructed by the pretrained data autoencoder are presented in App.~\ref{app:twostage}.}. 
To quantify the crossing points in trajectories, we train all models with the linear dynamics assumption and measure the proportion of samples where the predicted labels from a one-step Euler solver and an adaptive-step solver disagree\footnote{
We also analyze the relationship between time $t$ and the occurrence of trajectory crossing in App.~\ref{app:where_crossing}.}. 
We report the disagreement ratio and training accuracy in Tab.~\ref{tab:twostage}.

As discussed in Sec.~\ref{sec:conflict}, our results indicate that merely augmenting the dimension (ANODE+FM) does not resolve the issue of crossing trajectories induced by the predefined dynamics, resulting in a poor performance even on training data.
Although employing more sophisticated encoders (Autoencoder+FM) partially reduces disagreements, it still fails to fit the training data properly.
Such crossing trajectories can be eliminated by learning encoders with the flow matching loss, allowing our model to fit successfully to training data with high accuracy.

\paragraph{Role of the Learned Dynamics Function}
A potential concern is that learning a nonlinear transformation to embed data might lead the data encoder to perfectly predict the label.
This issue, noted in previous work~\cite{dissecting}, warns that learning a complex input transformation could result in a collapse where the dynamics function becomes a simple identity map.
To investigate whether this issue occurs in our case, we conduct a 1-NN classification using the learned data encoder on CIFAR10 image classification.
The 1-NN classification accuracy with the learned data embedding $z_0$ is 65.66\%, which is significantly lower than the accuracy of 88.47\% when we utilize the learned dynamics function to obtain the predicted label embedding $\hat{z}_1$.
This result indicates that the learned data embedding is not yet linearly separable enough, and the learned flow can further enhance accuracy.
Based on this observation, we conclude that the collapse scenario, where the velocity field becomes an identity map, does not occur in our model.
Instead, we find that the learned dynamics function clearly plays a role in processing the data, thereby avoiding the aforementioned pitfall.

\begin{wrapfigure}{r}{0.45\textwidth}
    \centering
    \vspace{-0.2cm}
    \includegraphics[width=0.44\textwidth]{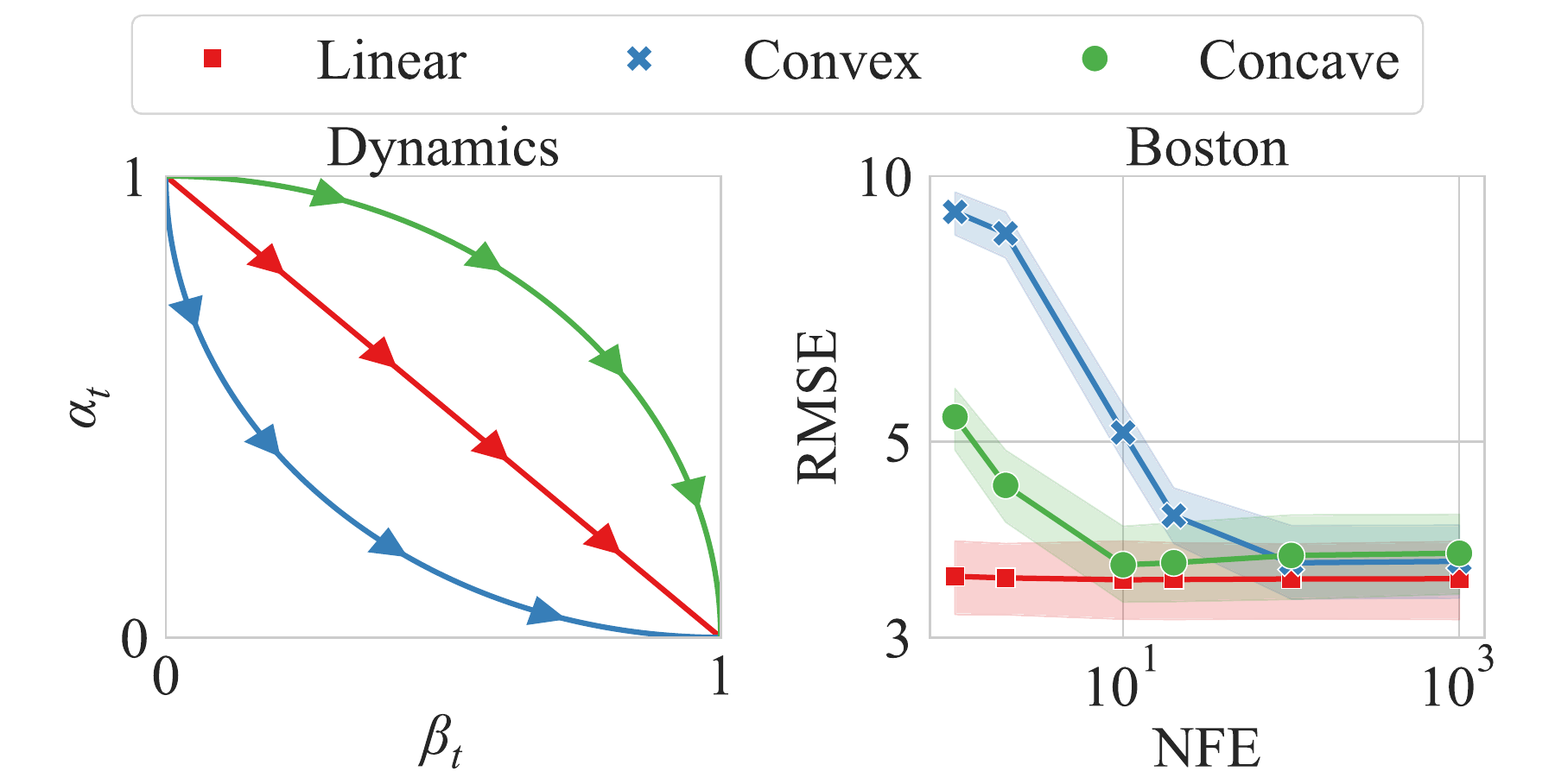}
    \caption{
    Analysis on predefined dynamics. 
    (Left) Change of coefficients in interpolant with respective to time. 
    (Right) Prediction RMSE over NFE on UCI Boston dataset.
    }
    \label{fig:dynamics_type_acc}
    \vspace{-1.2cm}
\end{wrapfigure}

\paragraph{Nonlinear Predefined Dynamics} \label{sec:dynamics_type}
Our method embraces not only the linear dynamics but also dynamics having fixed form of $z_t=\alpha_t z_0 + \beta_t z_1$ in general.
Here, we investigate the effect of utilizing different dynamics assumptions, including nonlinear ones.
To be specific, we choose three different dynamics that is easy to implement:
\begin{itemize}[leftmargin=*]
    \item Concave: $\alpha_t=\cos(\frac{\pi}{2}t), \beta_t=\sin(\frac{\pi}{2}t)$ 
    \item Linear: $\alpha_t=(1-t), \beta_t=t$
    \item Convex: $\alpha_t=1-\sin(\frac{\pi}{2}t), \beta_t=1-\cos(\frac{\pi}{2}t)$.
\end{itemize}
We visualize the difference of dynamics and their effect on final performance in Fig.~\ref{fig:dynamics_type_acc}. 

As it shows, the performance of the nonlinear variants (\emph{i.e.}, convex and concave) clearly improves with an increased number of function evaluations. 
For these variants, the many-step inference performance increases as we invest more NFEs and then later saturates, indicating that the model's approximation of an infinite-depth model becomes sufficiently accurate as discretization error diminishes.
In contrast, our model trained on the linear dynamics shows consistently good performance, even with a few function evaluations.
This behavior is somewhat expected, as the ideal linear trajectory can be already accurately inferred using a one-step Euler solver.
Surprisingly, we also empirically observe that the choice of linear dynamics leads to better performance, compared to more complex choice of dynamics assumption.
Thus, similar to the claims of Liu et al.(2022)~\cite{liu2022rectifiedflow}, we believe that linear dynamics should be considered as a default choice unless specific constraints on hidden states are required.

\paragraph{Ablation on Optimization Techniques} \label{ablation_collapse}
\begin{figure}
    \centering
    \includegraphics[width=1.0\linewidth]{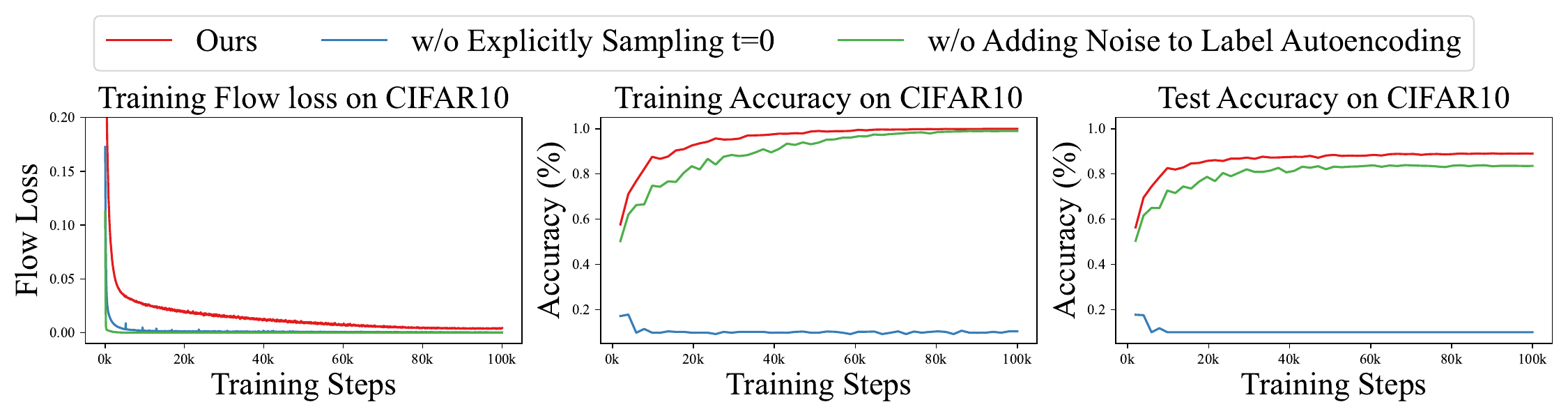}
    \caption{Ablation study of optimization techniques on CIFAR10. Explicitly sampling $t=0$ in training prevents suboptimal solutions while adding noise to label autoencoding improves generalization.}
    \label{fig:ablation_collapse}
\end{figure}

We conduct an ablation study on CIFAR10 dataset to investigate the effects of optimization techniques introduced in Sec.~\ref{optimization_technique}, and report the results in Fig.~\ref{fig:ablation_collapse}. 
Similar ablation studies were also conducted on other datasets, as detailed in App.~\ref{app:more_ablation}.
The variant without explicitly sampling $t=0$ fails to fit on training data, despite the convergence of the flow loss.
Furthermore, the variant without adding noise to the label autoencoding objective succeeds in fitting the training set, but its test accuracy significantly degrades compared to the version with noise in label autoencoding. 
Thus, as discussed in Sec.~\ref{optimization_technique}, we conclude that explicitly sampling $t=0$ and introducing noise in label autoencoding effectively regularize both encoders and produces label embedding that is robust to test-time errors.

\section{Limitations and Future Work} \label{sec:limitation}
In this work, we primarily study the problem of adopting flow matching by imposing a fixed and simple dynamics assumption, observing that a fixed, closed-form equation for the intermediate state is already sufficient to bring the advantages of flow matching.
Although employing simple dynamics (\emph{e.g.} linear) may seem overly restrictive, this approach can be justified in the context of Koopman operator theory, which aims to find embeddings that globally linearize the dynamics. Since our method of flow matching with linear dynamics shares a high-level concept with Koopman operator theory, exploring the relationship between the two could be a promising direction for future research. We leave further discussions in App.~\ref{app:koopoman}.

On the other hand, while we studied simple and predefined dynamics, the form of dynamics assumption could be further generalized to be a learnable component.
To be specific, introducing a learnable target dynamics that determines per-sample dynamics would be a promising future direction to study.
This would be advantageous for handling inputs with varying complexity, by efficiently and adaptively allocating more computation on hard samples.

Furthermore, extending our method to a broader range of paired data applications might be a useful future direction. 
The reduced computational burden achieved through simulation-free training could offer several benefits of continuous-depth models to diverse applications. 
For example, applying our method to model compression or knowledge distillation could be particularly promising, leveraging the parameter efficiency of continuous-depth models.

\section{Conclusion}
In this work, we adopted a flow matching objective to achieve simulation-free training of continuous-depth models for learning deterministic mappings between paired data. 
We proposed learning an embedding space where flow matching occurs, which we identified as a crucial component for ensuring the validity of the target velocity field.
Our proposed method significantly reduces the computational burden of training NODEs while maintaining competitive performance. 
Additionally, we found that our method, leveraging simple linear dynamics, demonstrates impressive performance on low-NFE regime.

\begin{ack}
This work was in part supported by the National Research Foundation of Korea (RS-2024-00351212 and RS-2024-00436165), and Institute of Information and communications Technology Planning and Evaluation (IITP) grant (RS-2022-II220926, RS-2024-00509279, RS-2021-II212068, and RS-2019-II190075) funded by the Korean government (MSIT).
We thank Youngmin Ryou (KAIST) for his valuable input in formulating the theoretical analysis presented in this work.

\end{ack}

\newpage
\bibliography{custom}
\bibliographystyle{acl_natbib}
\newpage
\appendix

\section*{Appendix}

\section{Proofs}
\label{app:theoretical_results}
\setcounter{prop}{0}
\setcounter{defin}{0}

In this section, we provide full proofs for the propositions in Sec.~\ref{sec:method}.
We first start with formal definition of target trajectory crossing.
Assume that we have a data encoder $f_\phi$ and label encoder $g_\varphi$ that transform data $x \in \mathbb{R}^{d_x}$ and labels $y\in \mathbb{R}^{d_y}$ to latent $z_0, z_1 \in \mathbb{R}^{d}$, where $z_0=f_\phi(x)$ and $z_1=g_\varphi(y)$, respectively.
We also choose a predefined dynamics $F({z}_0,{z}_1,t) = \alpha_t {z}_0+ \beta_t{z}_1 ={z}_t$.
Under the assumption of $d>d_x, d_y$\footnote{ 
As $d_x$ and $d_y$ are dimensions in observation space, we also conjecture that the latent dimension $d$ can be made smaller if the data lives on a low-dimensional manifold.
} and both $\alpha_t$, $\beta_t$ being smooth and nonzero except for $t=0$ and $t=1$, we define the target trajectory crossing as follows:
\begin{defin}[Target Trajectory Crossing]
    The encoders $(f_\phi,g_\varphi)$ are said to induce a target trajectory crossing if there exists a tuple $(t,{x},{y},{x}^\prime,{y}^\prime)$ such that  $\alpha_t f_\phi({x})+ \beta_t g_\varphi({y})=\alpha_t f_\phi({x}^\prime)+ \beta_t g_\varphi({y}^\prime)$ for ${x}\neq{x}^\prime$ and ${y}\neq{y}^\prime$.
\end{defin}
We now provide the proofs for each proposition in the following subsections.

\subsection{Proof of Proposition~\ref{prop:1} (Section~\ref{sec:method})}
\label{app:proof_prop1}
\begin{prop}
    There exist $(f_\phi,g_\varphi)$ that induces non-crossing target trajectory for any data pair in $\mathcal{D}$ while minimizing $\mathcal{L}_{label\_ae}$.
\end{prop}
\begin{proof}
    Let the latent space constructed by a set of basis $\mathbb{I}=\{{e}_1,{e}_{2},...,{e}_d\}$. Since $d>d_y$, we can find a label encoder $g_\varphi$ such that utilizes $k$ basis $\mathbb{J}=\{{e}_1,{e}_{2},...,{e}_k\}$ ($d>k\geq d_y$ ) and minimizes the autoencoding loss (\emph{i.e.}, $g_\varphi({y})=g_\varphi({y}')$ if and only if ${y}={y}'$). Also, we can find a data encoder $f_\phi$  such that $\text{proj}_{\text{span}(\mathbb{K})}f_\phi({x}) = \text{proj}_{\text{span}(\mathbb{K})}f_\phi({x}')$ if and only if ${x}={x}'$, where $\mathbb{K}=\{{e}_{k+1}, ..., {e}_{d}\}$.

    Then, suppose that there exists a tuple $(t,{x},{y},{x}^\prime,{y}^\prime)$ such that $\alpha_t f_\phi({x})+ \beta_t g_\varphi({y})=\alpha_t f_\phi({x}^\prime)+ \beta_t g_\varphi({y}^\prime)$, \emph{i.e.}, $\alpha_t (f_\phi({x})-f_\phi({x}'))+ \beta_t( g_\varphi({y})- g_\varphi({y}'))= {0}$.
    
    Since $g_\varphi({y})- g_\varphi({y}') ={0}$ if and only if ${y} = {y}'$ and $\text{proj}_{\text{span}({\mathbb{K})}}(f_\phi({x})-f_\phi({x}'))={0}$ if and only if ${x}={x}'$ by construction, such a tuple does not exist. Therefore, there exists $f_\phi, g_\varphi$ such that does not induce target trajectory crossing, while minimizing the autoencoding loss.
    \end{proof}

\subsection{Proof of Proposition~\ref{prop:2} (Section~\ref{sec:method})}
\label{app:proof_prop2}
\begin{prop}
    If $g_\varphi$ is injective, the following equivalence holds: $(f_\phi,g_\varphi, h_\theta)$ minimizes $\mathcal{L}_{flow}$ to zero for all $t\in[0, 1)$ if and only if $(f_\phi,g_\varphi)$ induces non-crossing target trajectory and $h_\theta$ perfectly fits the induced target velocity, for any data pair in $\mathcal{D}$.
\end{prop}
\begin{proof}
    ($\Longleftarrow$) If $(f_\phi,g_\varphi)$ induces non-crossing target trajectory for any data pair in $\mathcal{D}$, there is a well-defined target velocity $\frac{d}{dt}{z}_t$ at every ${z}_t$ which is continuous on $t$. If $h_\theta$ perfectly fits this target velocity for all $({z}_t, t)$, the flow loss is zero.

    ($\Longrightarrow$) We prove by contradiction. Suppose the flow loss is zero but there is a crossing trajectory, \emph{i.e.}, there exists a tuple $(t,{x},{y},{x}^\prime,{y}^\prime)$ that ${z}_t={z}'_t$ for ${x}\neq{x}^\prime$ and ${y}\neq{y}^\prime$. Since the loss is zero for all $t\in[0, 1)$, the dynamics function $h_\theta$ must output $\frac{d}{dt}F({z}_0,{z}_1,t)$ at ${z}_t$, and  $\frac{d}{dt}F({z}'_0,{z}'_1,t)$ at ${z}'_t$. This is a contradiction since at the point of crossing we have ${z}_t={z}'_t$ but $\frac{d}{dt}F({z}_0,{z}_1,t)\neq\frac{d}{dt}F({z}'_0,{z}'_1,t)$.
\end{proof}

\section{Experiment Details}
\label{app:experiment_details}
This section describes the implementation detail in our experiments (Sec.~\ref{sec:experiments}).

\subsection{Experiment Details for Classification Tasks}
\paragraph{Network Architecture (MNIST)}
We employ an MLP-based architecture for the MNIST dataset.
The data encoder maps each image to a 784-dimensional embedding, using a three-layer MLP with a hidden dimension of 784 and BatchNorm~\cite{IoffeS15batchnorm}.
The dynamics function consists of a three-layer MLP with 2048 hidden dimension.
Following NODE~\cite{chen2018neuralode}, we concatenate the time variable to the input of each layer.
Consistent with common practices for NODEs~\cite{kidger2022onnodes}, normalization layers are not included in the dynamics function.
For the label autoencoder, class labels are converted into one-hot vectors and encoded with a single linear layer, while the label decoder utilizes a two-layer MLP with BatchNorm.

\paragraph{Network Architecture (SVHN, CIFAR10)} We adopt a CNN-based architecture for both the SVHN and CIFAR10 datasets.
The data encoder utilizes 7 convolutional layers with a hidden dimension of 64.
By default, we employ 3x3 convolution kernels, while the final two layers utilize 4x4 kernels with a stride of 2 for downsampling, resulting in the data being encoded into states with dimensions of 7x7x64.
The dynamics function consists of 6 convolutional layers with 3x3 kernels and a hidden dimension of 256.
At each layer, the time variable is concatenated to the input.
We utilize a single linear layer for label encoding, reshaping the output to match the size of embedding. In the label decoding process, we average the feature map over the spatial dimension and apply a single linear layer.

\begin{wrapfigure}{r}{0.45\textwidth}
    \centering
    \includegraphics[width=0.45\textwidth]{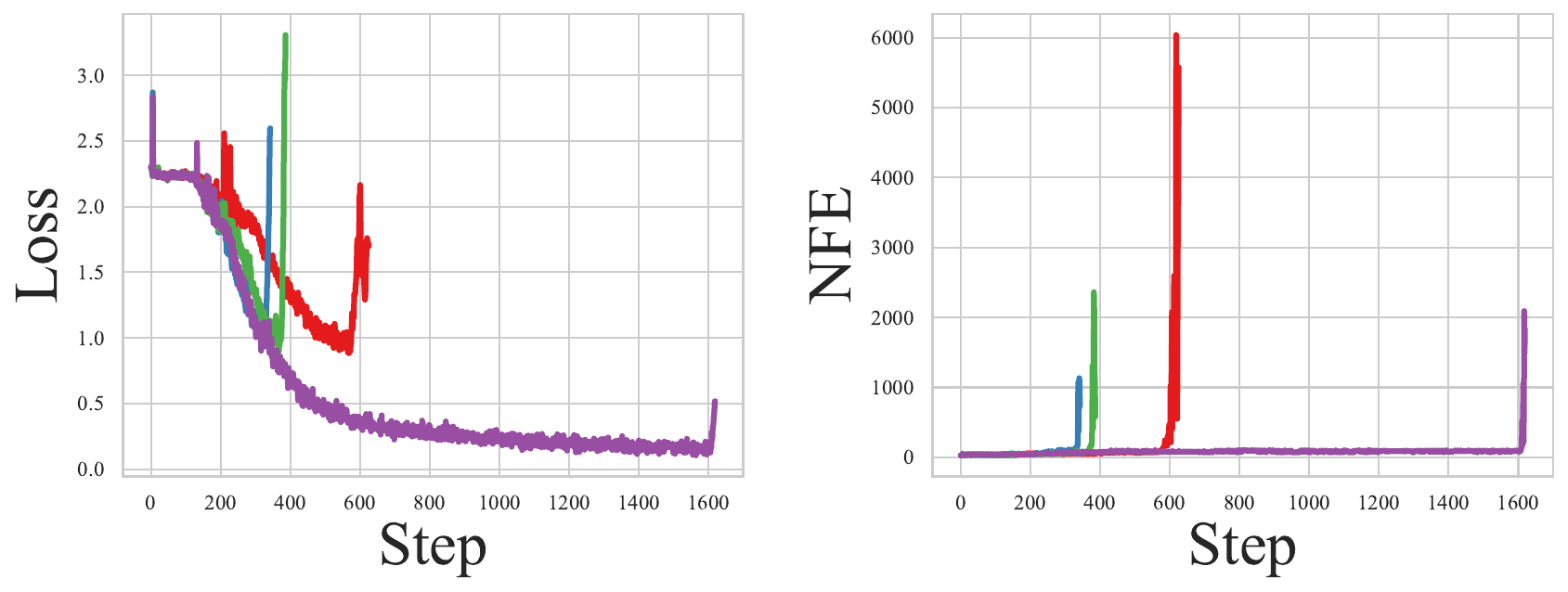}
    \caption{Failure cases of NODEs.}
    \label{fig:node_failure}
\end{wrapfigure}

\paragraph{Training} We train all models for 100,000 iterations using the Adam optimizer~\cite{KingmaB14adam} with a cosine learning rate scheduler.
For all classification experiments, we utilize a batch size of 1024.
Additionally, we set a maximum training time of 48 hours for the feasibility of experiments.
By default, we set the learning rate to 1e-3 for MNIST and 3e-4 for CIFAR10 and SVHN.
For our method, we set the ratio of explicitly sampling $t=0$ to 10\% for all datasets. Regarding the noise introduced to the label autoencoder, we set the standard deviation $\sigma$ to 3 for MNIST, 7 for SVHN, and 10 for CIFAR10.
In cases where training of NODE baselines fails, we adjust the learning rate accordingly.
The failure cases of NODE baselines are illustrated in Fig.~\ref{fig:node_failure}.
\subsection{Experiment Details for Regression Tasks}
\paragraph{Network Architecture} We utilize a 2-layer MLP with a hidden dimension of 64 for the data encoder, and a 7-layer MLP with the same hidden dimension for the dynamics function. In the dynamics function, we concatenate the time variable to the input of each layer. On the label side, we employ a single linear layer for encoding and another single linear layer for decoding.

\paragraph{Training} We trained all models for 100,000 iterations using the Adam optimizer with a constant learning rate of 3e-3. 
Additionally, we split the training set into a train-validation split with a ratio of 6:4, and utilized the validation metric for early-stopping.
Early stopping was implemented by measuring the validation metric every 1,000 iterations and setting the patience level to 10.
For our model, we sample the noise added to the label autoencoding from $\mathcal{N}(0, 3^2)$ and the proportion of explicitly sampling $t=0$ as 10\%.

\subsection{Experiment Details about Baselines}
\paragraph{NODE~\cite{chen2018neuralode}} Following the tolerance values used in NODE~\cite{chen2018neuralode} and ANODE~\cite{dupont2019anode}, we used dopri5~\cite{dopri5} solver with the absolute and relative tolerance of 1e-3 for both training and inference.

\paragraph{STEER~\cite{arnab2020steer}} 
STEER introduces a new hyperparameter $b$ that controls the integration interval. Instead of integrating the dynamics function from 0 to 1, STEER integrates from 0 to $t \sim \text{Uniform}(1-b, 1+b)$ during training. We follow the default configuration of using $b=0.99$ for MNIST classification. For other tasks, we set $b$ to 0.1 since higher values resulted in training failures.

\paragraph{RNODE~\cite{finlay2020rnode}} RNODE introduces two hyperparameters used for the coefficients of regularization terms.
Specifically, it regularizes the Jacobian norm and the kinetic energy of the dynamics function to encourage the model to learn straight and constant-speed dynamics.
The coefficient of 0.01 was generally used throughout the experiments in the original paper.
Therefore, we set both coefficients for the Jacobian norm and kinetic energy as 0.01 for our experiments.

\paragraph{CARD~\cite{han2022card}} CARD utilizes discrete timesteps and a hyperparameter $\beta_t$ to schedule the noise level for diffusion modeling. Following the paper's approach, we use 1000 discrete timesteps and set a linear noise schedule from $\beta_1 = $1e-4 to $\beta_{1000} = 0.02$.

\subsection{Computation Resources}
We conducted experiments on our internal cluster with two types of machine.
We list their specifications below.
\begin{enumerate}
    \item Intel Xeon Gold 6330 CPU and NVIDIA RTX A6000 GPU (with 48GB VRAM)
    \item Intel Xeon Gold 6230 CPU and NVIDIA RTX 3090 GPU (with 24GB VRAM)
\end{enumerate}
We utilized the first machine for image classification experiments, and latter one for regression experiments.
We expect training our model will take about 90 min., 270 min, 230 min. for classification experiments on MNIST, SVHN and CIFAR10, respectively.
For regression tasks, we estimate training cost, summing up for all splits and datasets, would cost about 288 GPU hours.

\section{Relation to Koopman Autoencoder}
\label{app:koopoman}
Our model with linear dynamics shares the high-level motivation with Lusch et al. (2018)~\cite{lusch2018deep}, which aims to find an embedding space that yields linear dynamics between source and target. 
Regardless of the theoretical background, both flow matching and Koopman operator theory are promising approaches that seek to interpret a nonlinear system within a well-studied linear framework.

At the same time, we identify several differences between our work and the line of research based on Koopman operator theory. 
While those works mainly focus on a systematic way to obtain a linearized representation of the underlying nonlinear dynamics (with eigenfunctions), our work aims to find a way to learn it in a simulation-free manner, avoiding the heavy computation of forward simulation (\emph{e.g.}, which appears in $\mathcal{L}_{lin}$ of Lusch et al. (2018)~\cite{lusch2018deep}) from an initial state to an end state. 
Additionally, compared to the discrete depth neural networks that have a single linear layer processor, our proposed method is generally applicable to any nonlinear dynamics that connects two endpoints $z_0$ and $z_1$, exemplified as \textit{convex} or \textit{concave} as discussed in Sec.~\ref{sec:dynamics_type}. 
This implies that in our case, it is possible to have a latent trajectory as a curve in non-Euclidean geometry whenever the interpolated state $z_t$ is tractable.

\section{Additional Results}
In this section, we extend the discussions in Sec.~\ref{sec:experiments} with additional results.

\subsection{Additional Result of Learning Encoders with Flow Loss}
\label{app:twostage}

\begin{wrapfigure}{r}{0.5\textwidth}
    \vspace{-0.5cm}
    \centering
    \includegraphics[width=0.45\textwidth]{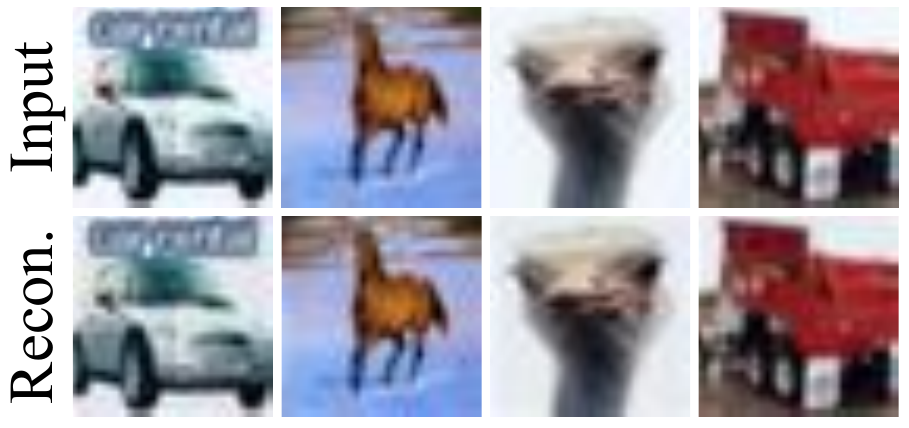}
    \vspace{-0.2cm}
    \caption{Reconstruction from the autoencoder.}
    \label{fig:reconstructed}
    \vspace{-0.3cm}
\end{wrapfigure}

We present the reconstructed images from the trained data autoencoder in Fig.~\ref{fig:reconstructed}, which are used to analyze the effectiveness of the flow loss as reported in Tab.~\ref{tab:twostage}. 
While the pretrained autoencoder reconstructs the images holistically with minimal information loss (Tab.~\ref{tab:twostage}), the embedding space learned without the flow loss fails to maintain the coupling with the predefined dynamics.

\subsection{Location of Crossing Points}
\label{app:where_crossing}

\begin{figure}[t]
    \centering
    \includegraphics[width=0.7\textwidth]{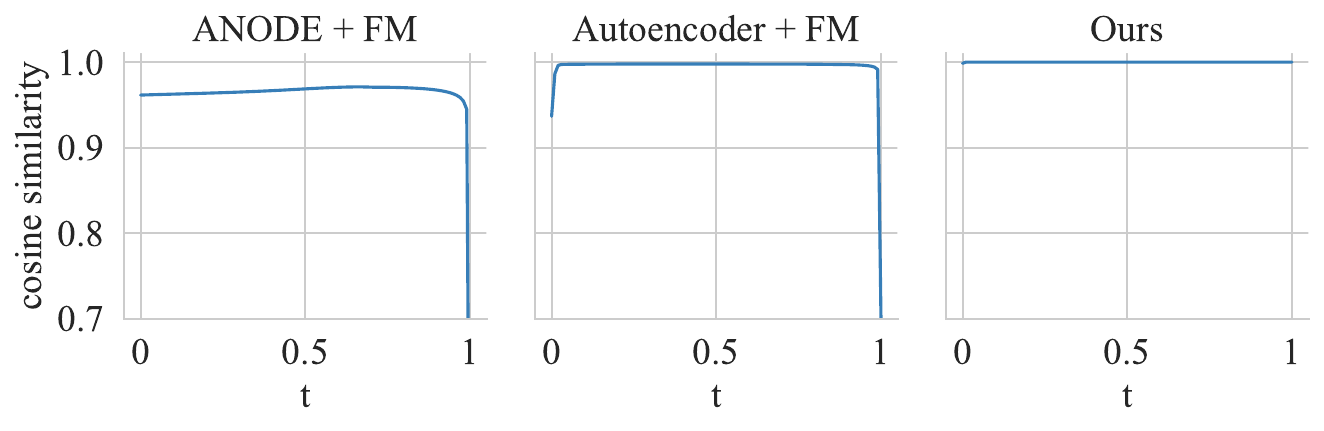}
    \vspace{-0.3cm}
    \caption{
    Cosine similarity between target velocity and predicted velocity over time $t$. Similar to the flow loss, we measure the cosine similarity between target velocity and predicted velocity. Low cosine similarity near $t=0$ and $t=1$ indicates the occurrence of target trajectory intersection near the endpoints.
    }
    \vspace{-0.4cm}
    \label{fig:where_crossing}
\end{figure}

In Fig.~\ref{fig:where_crossing}, we extend the analysis in Sec.~\ref{sec:analysis} to identify where target trajectory intersections occur. 
Specifically, we measure the cosine similarity between target velocity and predicted velocity across $t$.
By measuring cosine similarity instead of MSE, we can ignore the effect of absolute scale in the embedding space and accurately compare ANODE+FM, Autoencoder+FM, and our method.
Model variants that do not learn encoders with flow loss particularly suffer from trajectory crossing near both endpoints, showing low prediction accuracy for the direction of target velocity.
In contrast, our proposed method shows consistently high cosine similarity, mitigating the issue of target trajectory intersections near these regions.
Since both endpoints are constructed from encoders, the result also supports our claim that the encoders should be trained with flow loss to effectively penalize such intersections.

\subsection{Additional Results of Ablation on Optimization Techniques} \label{app:more_ablation}
We provide additional results on the optimization techniques in Fig.~\ref{fig:ablate_more}. Consistent with the findings on CIFAR10 (Fig.~\ref{fig:ablation_collapse}), the model without explicitly sampling at $t=0$ converges to suboptimal solutions, while introducing noise during label autoencoding improves test accuracy.

\begin{figure}[t]
    \centering
    \includegraphics[width=1.0\textwidth]{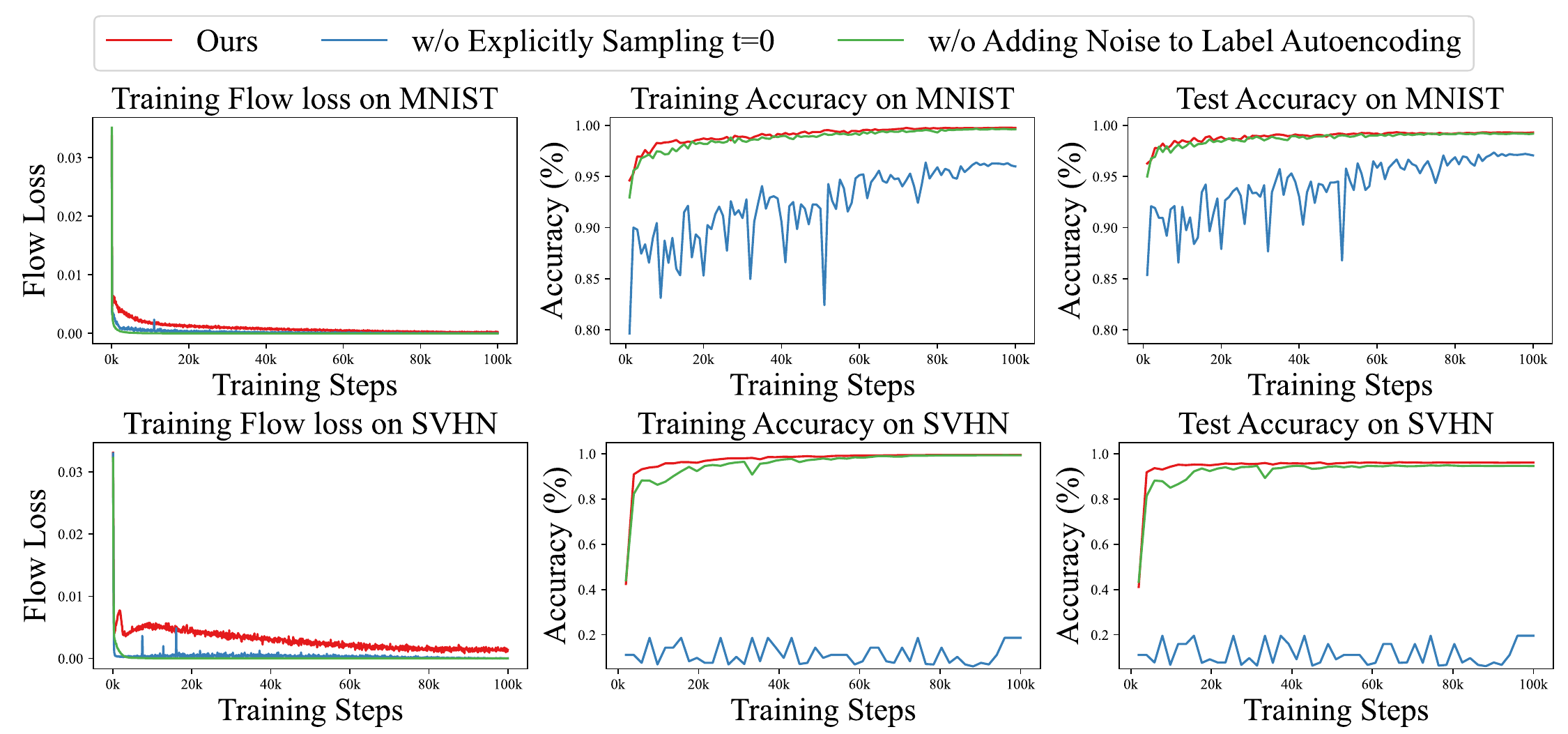}
    \vspace{-0.3cm}
    \caption{Ablation study of optimization techniques on MNIST and SVHN datasets, corresponding to Fig.~\ref{fig:ablation_collapse}. Explicitly sampling $t=0$ prevents suboptimal solutions, while adding noise to label autoencoding improves generalization performance.}
    \vspace{-0.2cm}
    \label{fig:ablate_more}
\end{figure}

\subsection{Additional Experiment Result for UCI Regression Tasks}
\label{app:regression_full}
The full results on UCI regression tasks with various solvers are shown in from Tab.~\ref{tab:UCI_1} to \ref{tab:UCI_dopri}.
As discussed in the main text, our method with linear dynamics assumption shows superior performance compared to baselines in low-NFE (1-2 steps) regime, which aligns with our observations in image classification experiments.

\begin{table}[t]
    \centering
    \caption{Experiment results on UCI regression tasks with Euler 1-step solver.}
    \label{tab:UCI_1}
    \footnotesize
    \begin{tabular}{c|ccccc}
    \toprule
    Dataset & Ours & NODE & STEER & RNODE & CARD \\
    \midrule
    Boston & 3.52$\pm$0.79 & 5.95$\pm$5.10 & 552.10$\pm$1638.54 & 3.50$\pm$0.84 & 130.23$\pm$38.91 \\
    Concrete & 5.59$\pm$0.58 & 2922.76$\pm$12615.38 & 1056.15$\pm$2607.27 & 5.98$\pm$1.66 & 198.91$\pm$72.58 \\
    Energy & 0.58$\pm$0.11 & 253.57$\pm$707.63 & 74.06$\pm$164.27 & 0.83$\pm$0.41 & 63.87$\pm$12.81 \\
    Kin8nm & 7.61$\pm$0.02 & 39.04$\pm$3.31 & 1694.59$\pm$700.19 & 23.69$\pm$0.82 & 238.18$\pm$6.44 \\
    Naval & 0.08$\pm$0.00 & 8447.16$\pm$2968.35 & 3011.13$\pm$907.95 & 0.16$\pm$0.01 & 15.42$\pm$0.59 \\
    Power & 4.03$\pm$0.19 & $>10^4$ & 7.37$\pm$14.16 & 4.06$\pm$0.35 & 179.91$\pm$65.08 \\
    Protein & 4.26$\pm$0.08 & 103.14$\pm$141.55 & 102.17$\pm$97.95 & 7.70$\pm$1.88 & 92.82$\pm$23.84 \\
    Wine & 0.67$\pm$0.05 & 1.99$\pm$4.31 & 19.08$\pm$36.97 & 0.67$\pm$0.05 & 10.42$\pm$2.87 \\
    Yacht & 0.92$\pm$0.42 & 62.19$\pm$108.46 & 29.06$\pm$41.13 & 2.71$\pm$1.94 & 111.21$\pm$62.53 \\
    Year & 9.15$\pm$NA & 1366.06$\pm$NA & 2452.93$\pm$NA & 22.37$\pm$NA & 157.11$\pm$NA \\
    
    \bottomrule
    \end{tabular}
\end{table}

\begin{table}[t]
    \centering
    \caption{Experiment results on UCI regression tasks with Euler 2-step solver.}
    \footnotesize
    \begin{tabular}{c|ccccc}
    \toprule
    Dataset & Ours & NODE & STEER & RNODE & CARD \\
    \midrule
    Boston & 3.50$\pm$0.78 & 3.81$\pm$2.04 & 278.23$\pm$827.35 & 3.41$\pm$0.81 & 20.86$\pm$12.26 \\
    Concrete & 5.58$\pm$0.57 & 1560.21$\pm$6301.49 & 1437.21$\pm$5229.38 & 5.59$\pm$0.80 & 28.49$\pm$17.08 \\
    Energy & 0.57$\pm$0.11 & 114.46$\pm$310.32 & 36.50$\pm$82.19 & 0.61$\pm$0.11 & 8.15$\pm$8.57 \\
    Kin8nm & 7.59$\pm$0.02 & 14.62$\pm$1.23 & 813.00$\pm$338.36 & 9.77$\pm$0.17 & 18.87$\pm$0.78 \\
    Naval & 0.07$\pm$0.00 & $>10^4$ & $>10^4$ & 0.08$\pm$0.00 & 1.81$\pm$0.22 \\
    Power & 4.03$\pm$0.19 & $>10^4$ & 5.46$\pm$6.52 & 3.86$\pm$0.23 & 19.78$\pm$10.59 \\
    Protein & 4.26$\pm$0.08 & 30.85$\pm$24.65 & 77.94$\pm$77.06 & 4.40$\pm$0.35 & 14.17$\pm$6.97 \\
    Wine & 0.67$\pm$0.05 & 3.50$\pm$10.84 & 24.95$\pm$46.23 & 0.66$\pm$0.04 & 1.49$\pm$0.64 \\
    Yacht & 0.93$\pm$0.43 & 34.66$\pm$62.06 & 57.13$\pm$131.04 & 1.79$\pm$1.01 & 11.83$\pm$19.40 \\
    Year & 9.17$\pm$NA & 6053.32$\pm$NA & 542.82$\pm$NA & 9.67$\pm$NA & 42.18$\pm$NA \\
    \bottomrule
    \end{tabular}
\end{table}

\begin{table}[t]
    \centering
    \caption{Experiment results on UCI regression tasks with Euler 10-step solver.}
    \footnotesize
    \begin{tabular}{c|ccccc}
    \toprule
    Dataset & Ours & NODE & STEER & RNODE & CARD \\
    \midrule
    Boston & 3.49$\pm$0.79 & 3.20$\pm$0.82 & 56.54$\pm$164.64 & 3.39$\pm$0.80 & 3.32$\pm$0.89 \\
    Concrete & 5.57$\pm$0.54 & 294.54$\pm$1260.67 & 88.00$\pm$190.75 & 5.59$\pm$0.77 & 5.47$\pm$0.67 \\
    Energy & 0.58$\pm$0.12 & 21.24$\pm$57.69 & 6.13$\pm$15.60 & 0.52$\pm$0.06 & 0.58$\pm$0.09 \\
    Kin8nm & 7.59$\pm$0.02 & 7.46$\pm$0.02 & 159.73$\pm$64.92 & 7.43$\pm$0.02 & 7.71$\pm$0.02 \\
    Naval & 0.07$\pm$0.00 & $>10^4$ & $>10^4$ & 0.03$\pm$0.00 & 0.09$\pm$0.00 \\
    Power & 4.03$\pm$0.18 & 2144.77$\pm$9543.08 & 4.10$\pm$0.71 & 3.82$\pm$0.21 & 4.57$\pm$0.19 \\
    Protein & 4.26$\pm$0.08 & 7.38$\pm$6.01 & 7.92$\pm$7.56 & 3.87$\pm$0.05 & 5.66$\pm$0.12 \\
    Wine & 0.68$\pm$0.05 & 0.82$\pm$0.49 & $>10^4$ & 0.66$\pm$0.04 & 0.75$\pm$0.05 \\
    Yacht & 0.96$\pm$0.45 & 2.58$\pm$6.37 & 1.86$\pm$1.93 & 1.21$\pm$0.60 & 0.95$\pm$0.34 \\
    Year & 9.25$\pm$NA & 139.85$\pm$NA & 69.53$\pm$NA & 8.87$\pm$NA & 10.85$\pm$NA \\
    \bottomrule
    \end{tabular}
\end{table}

\begin{table}[t]
    \centering
    \caption{Experiment results on UCI regression tasks with Euler 20-step solver.}
    \footnotesize
    \begin{tabular}{c|ccccc}
    \toprule
    Dataset & Ours & NODE & STEER & RNODE & CARD \\
    \midrule
    Boston & 3.49$\pm$0.79 & 3.20$\pm$0.82 & 29.29$\pm$80.54 & 3.39$\pm$0.80 & 3.32$\pm$0.89 \\
    Concrete & 5.58$\pm$0.55 & 148.43$\pm$629.86 & 42.13$\pm$85.51 & 5.59$\pm$0.77 & 5.43$\pm$0.66 \\
    Energy & 0.58$\pm$0.13 & 10.43$\pm$28.28 & 2.63$\pm$7.20 & 0.52$\pm$0.06 & 0.53$\pm$0.09 \\
    Kin8nm & 7.62$\pm$0.02 & 7.39$\pm$0.02 & 81.51$\pm$31.91 & 7.38$\pm$0.02 & 7.95$\pm$0.03 \\
    Naval & 0.07$\pm$0.00 & $>10^4$ & $>10^4$ & 0.03$\pm$0.00 & 0.09$\pm$0.00 \\
    Power & 4.04$\pm$0.18 & 1073.63$\pm$4769.25 & 3.97$\pm$0.19 & 3.82$\pm$0.21 & 4.81$\pm$0.18 \\
    Protein & 4.27$\pm$0.08 & 5.55$\pm$2.65 & 5.33$\pm$2.38 & 3.86$\pm$0.04 & 5.90$\pm$0.07 \\
    Wine & 0.68$\pm$0.05 & 0.71$\pm$0.17 & 1.10$\pm$0.72 & 0.66$\pm$0.04 & 0.75$\pm$0.06 \\
    Yacht & 0.96$\pm$0.46 & 1.29$\pm$0.90 & 1.35$\pm$0.67 & 1.17$\pm$0.60 & 0.89$\pm$0.36 \\
    Year & 9.30$\pm$NA & 40.97$\pm$NA & 33.25$\pm$NA & 8.86$\pm$NA & 11.40$\pm$NA \\
    \bottomrule
    \end{tabular}
\end{table}

\begin{table}[t]
    \centering
    \caption{Experiment results on UCI regression tasks with Euler 100-step solver.}
    \footnotesize
    \begin{tabular}{c|ccccc}
    \toprule
    Dataset & Ours & NODE & STEER & RNODE & CARD \\
    \midrule
    Boston & 3.49$\pm$0.80 & 3.20$\pm$0.82 & 7.58$\pm$13.13 & 3.40$\pm$0.80 & 3.44$\pm$0.97 \\
    Concrete & 5.63$\pm$0.55 & 33.72$\pm$124.75 & 10.81$\pm$13.46 & 5.60$\pm$0.77 & 5.59$\pm$0.61 \\
    Energy & 0.59$\pm$0.14 & 2.12$\pm$4.92 & 0.64$\pm$0.52 & 0.52$\pm$0.06 & 0.55$\pm$0.10 \\
    Kin8nm & 7.77$\pm$0.02 & 7.37$\pm$0.02 & 20.30$\pm$5.69 & 7.36$\pm$0.03 & 8.38$\pm$0.02 \\
    Naval & 0.07$\pm$0.00 & 2095.27$\pm$667.36 & 116.16$\pm$43.60 & 0.03$\pm$0.00 & 0.12$\pm$0.00 \\
    Power & 4.08$\pm$0.16 & 216.78$\pm$950.18 & 3.95$\pm$0.18 & 3.82$\pm$0.21 & 5.03$\pm$0.18 \\
    Protein & 4.30$\pm$0.09 & 4.38$\pm$0.29 & 4.20$\pm$0.06 & 3.86$\pm$0.04 & 6.13$\pm$0.09 \\
    Wine & 0.68$\pm$0.05 & 0.66$\pm$0.04 & 0.71$\pm$0.07 & 0.66$\pm$0.04 & 0.78$\pm$0.06 \\
    Yacht & 0.96$\pm$0.46 & 1.12$\pm$0.54 & 1.23$\pm$0.48 & 1.16$\pm$0.60 & 0.98$\pm$0.38 \\
    Year & 9.35$\pm$NA & 8.94$\pm$NA & 9.41$\pm$NA & 8.86$\pm$NA & 11.87$\pm$NA \\
    \bottomrule
    \end{tabular}
\end{table}
\begin{table}[t]
    \centering
    \caption{Experiment results on UCI regression tasks with Euler 1000-step solver.}
    \footnotesize
    \begin{tabular}{c|ccccc}
    \toprule
    Dataset & Ours & NODE & STEER & RNODE & CARD \\
    \midrule
    Boston & 3.50$\pm$0.80 & 3.20$\pm$0.82 & 3.44$\pm$0.69 & 3.40$\pm$0.80 & 3.34$\pm$0.90 \\
    Concrete & 5.66$\pm$0.55 & 8.29$\pm$11.12 & 5.70$\pm$0.58 & 5.60$\pm$0.77 & 5.56$\pm$0.52 \\
    Energy & 0.59$\pm$0.14 & 0.52$\pm$0.08 & 0.53$\pm$0.07 & 0.52$\pm$0.06 & 0.58$\pm$0.07 \\
    Kin8nm & 7.87$\pm$0.02 & 7.36$\pm$0.02 & 7.69$\pm$0.06 & 7.36$\pm$0.03 & 8.53$\pm$0.03 \\
    Naval & 0.07$\pm$0.00 & 19.23$\pm$5.63 & 10.21$\pm$3.75 & 0.03$\pm$0.00 & 0.13$\pm$0.00 \\
    Power & 4.14$\pm$0.14 & 24.24$\pm$90.84 & 3.95$\pm$0.18 & 3.82$\pm$0.21 & 5.15$\pm$0.12 \\
    Protein & 4.36$\pm$0.09 & 4.33$\pm$0.23 & 4.20$\pm$0.06 & 3.86$\pm$0.04 & 6.17$\pm$0.11 \\
    Wine & 0.69$\pm$0.05 & 0.66$\pm$0.04 & 0.70$\pm$0.05 & 0.66$\pm$0.04 & 0.76$\pm$0.06 \\
    Yacht & 0.96$\pm$0.46 & 1.12$\pm$0.54 & 1.23$\pm$0.48 & 1.16$\pm$0.60 & 0.95$\pm$0.34 \\
    Year & 9.39$\pm$NA & 8.95$\pm$NA & 8.95$\pm$NA & 8.86$\pm$NA & 12.05$\pm$NA \\
    \bottomrule
    \end{tabular}
\end{table}

\begin{table}[ht!]
    \centering
    \caption{Experiment results on UCI regression tasks with dopri5 solver.}
    \footnotesize
    \begin{tabular}{c|ccccc}
    \toprule
    Dataset & Ours & NODE & STEER & RNODE & CARD \\
    \midrule
    Boston & 3.50$\pm$0.80 & 3.20$\pm$0.81 & 3.44$\pm$0.69 & 3.40$\pm$0.80 & 3.34$\pm$0.90 \\
    Concrete & 5.67$\pm$0.55 & 5.76$\pm$0.69 & 5.70$\pm$0.57 & 5.60$\pm$0.77 & 5.56$\pm$0.52 \\
    Energy & 0.60$\pm$0.14 & 0.52$\pm$0.08 & 0.53$\pm$0.07 & 0.51$\pm$0.06 & 0.58$\pm$0.07 \\
    Kin8nm & 7.89$\pm$0.02 & 7.35$\pm$0.02 & 7.56$\pm$0.02 & 7.37$\pm$0.03 & 8.53$\pm$0.03 \\
    Naval & 0.07$\pm$0.00 & 0.05$\pm$0.00 & 0.10$\pm$0.00 & 0.03$\pm$0.00 & 0.13$\pm$0.00 \\
    Power & 4.17$\pm$0.14 & 3.91$\pm$0.19 & 3.95$\pm$0.17 & 3.83$\pm$0.21 & 5.15$\pm$0.12 \\
    Protein & 4.37$\pm$0.11 & 4.33$\pm$0.23 & 4.20$\pm$0.06 & 3.86$\pm$0.04 & 6.17$\pm$0.11 \\
    Wine & 0.69$\pm$0.05 & 0.66$\pm$0.04 & 0.70$\pm$0.05 & 0.66$\pm$0.04 & 0.76$\pm$0.06 \\
    Yacht & 0.96$\pm$0.46 & 1.12$\pm$0.54 & 1.23$\pm$0.48 & 1.16$\pm$0.60 & 0.95$\pm$0.34 \\
    Year & 9.41$\pm$NA & 8.95$\pm$NA & 8.95$\pm$NA & 8.88$\pm$NA & 12.05$\pm$NA \\
    \bottomrule
    \end{tabular}
    \label{tab:UCI_dopri}
\end{table}

\clearpage

\end{document}